\documentclass[letterpaper, 10 pt, conference]{ieeeconf}  

\IEEEoverridecommandlockouts                              

\usepackage[noadjust]{cite}
\usepackage{graphicx}
\graphicspath{{./fig/}}

\usepackage{amsmath, amssymb}
\usepackage[caption=false,font=footnotesize]{subfig}
\usepackage{epsfig}

\newtheorem{theorem}{Theorem}
\newtheorem{remark}{Remark}

\usepackage{multicol,multirow}
\usepackage{xcolor}
\usepackage{epstopdf}
\usepackage{svg}

\DeclareMathOperator{\diag}{diag}
\DeclareMathOperator{\sat}{sat}
\usepackage{algorithm,algpseudocode}

\usepackage{hyperref} 
\usepackage{booktabs} 
\usepackage{array}    
\usepackage{multirow} 
\usepackage{geometry} 
\begin{document}

\title{Adaptive Artificial Time-Delay Control with Barrier Lyapunov Constraints for Euler–Lagrange Robots}

\author{
    Saksham Gupta$^{1}$, Rishabh Dev Yadav$^{2}$, Sarthak Mishra$^{1}$, Amitabh Sharma$^{1}$, Sourish Ganguly$^{1}$,\\ Wei~Pan$^{3}$, Spandan~Roy$^{1}$~and~Simone Baldi$^{4}$
    \thanks{The work is partly supported by the UASAT project sponsored by MeiTY, India. This work was partially supported by Jiangsu Provincial Scientific Research Center of Applied Mathematics No. BK20233002}
    \thanks{$^{1}$ The authors are with Robotics Research Center, International Institute of Information Technology Hyderabad, India.
     emails: \texttt{\{sakshammgupta, sourishganguly96\}@gmail.com}, \texttt{\{sarthak.mishra, amitabh.sharma\}@research.iiit.ac.in},  \texttt{spandan.roy@iiit.ac.in}}%
    \thanks{$^{2}$R. D. Yadav is with Department of Computer Science, University of Manchester, UK.
     email: \texttt{rishabh.yadav@postgrad.manchester.ac.uk}}%
     \thanks{$^{3}$W. Pan is with Autonomous Systems and Automatic Control in School of Engineering, Newcastle University, UK.
     email: \texttt{wei.pan2@newcastle.ac.uk}}
     \thanks{$^{4}$S. Baldi is with Self-Organizing Mobility Lab, School of Mathematics, Southeast University, Nanjing 210096, China, email:  \texttt{simonebaldi@seu.edu.cn}
}}

\maketitle

\begin{abstract}
We addresses the challenge of simultaneously compensating for state-dependent uncertainties and enforcing time-varying state constraints in Euler–Lagrange systems—common requirements in robotics yet underserved by existing control designs. A novel adaptive control framework is developed that combines an artificial time delay based uncertainty estimation (aka time delay estimation, TDE) with a barrier Lyapunov function to enforce constraint-aware control design. A state-dependent upper bound on the TDE approximation error is formulated, and an adaptive law is constructed to estimate its parameters online, enabling real-time state-dependent uncertainty compensation without relying on prior model knowledge. To ensure constraint compliance, the barrier Lyapunov function-based controller enforces time-varying bounds on both position and velocity. The resulting architecture is provably stable via Lyapunov analysis. Experimental results on a five degrees-of-freedom robotic manipulator validate the framework’s capability over the state-of-the-art in adhering to safety-critical constraints under dynamic uncertainties.
\end{abstract}



\section{Introduction}
Robotic systems are increasingly deployed in dynamic and uncertain environments, where safety, stability, and high performance are essential. Achieving reliable operation in such settings is particularly challenging due to the nonlinear Euler–Lagrange (EL) dynamics of the robots, and due to the effects of uncertain parameters and unmodeled terms. In these contexts, control strategies must operate reliably despite incomplete system knowledge, while enforcing state constraints to ensure task feasibility and safety.

Artificial time delay control (also known as time-delay estimation or TDE-based control) was introduced to reduce reliance on system models by approximating unknown dynamics using delayed state/input measurements \cite{103479, 4789852}. Its simplicity and low computational cost have led to widespread adoption in robotics \cite{7244191, 7506280,   8128491, 8701563, dhadekar2021robust, lim2019delayed}. However, most TDE-based methods \cite{7244191, 7506280,   8128491, 8701563, dhadekar2021robust, lim2019delayed} assume an a priori known bound on the approximation error—an assumption that limits robustness to state-dependent uncertainties. To overcome this limitation, adaptive TDE control frameworks have been proposed by \cite{9001188, 10011673, 9325954}, wherein a structured state-dependent upper bound on the TDE error is formulated and its parameters are updated online via adaptive laws. 
However, these methods do not incorporate mechanisms for enforcing state constraints, making them unsuitable for safety-critical applications.

In many practical applications, ensuring that system states remain within prescribed bounds is as critical as achieving tracking accuracy. Robotic arms operating near humans, aerial vehicles navigating constrained spaces, or manipulators handling delicate payloads, all demand strict enforcement of position and velocity limits for safety and reliability. While Model Predictive Control (MPC) can tackle such constraints, it requires accurate model and imposes high computational overhead  \cite{dhar2021indirect, berberich2020data, grune2017nonlinear, 10644611}. Barrier Lyapunov Functions (BLFs), by contrast, encode constraints directly into the control design, making it more unsuitable for real-time settings \cite{tee2009barrier, liu2016barrier, cao2022practical, sankaranarayanan2022robustifying, ganguly2022robust}.

Although early BLF-based control strategies required precise knowledge of system dynamics, more recent designs have relaxed this dependency by incorporating robust \cite{9636323, xu2021tangent} or adaptive \cite{liu2017barrier, shao2021adaptive, yang2020state, obeid2018barrier, laghrouche2021barrier} mechanisms. However, the majority of these methods enforce position-only 
reduce flexibility (cf. \cite{tee2009barrier, liu2016barrier, xu2021tangent, shao2021adaptive}). Time-varying BLF formulations have been developed (\cite{liu2020finite, ding2021adaptive, liu2019adaptive, fuentes2020adaptive, sankaranarayanan2022robustifying}), for improved constraint scheduling. Nevertheless, many of these approaches still rely on accurate model information (cf. \cite{liu2020finite, ding2021adaptive, sankaranarayanan2022robustifying}) or assume a priori bounded uncertainty (cf. \cite{liu2019adaptive, fuentes2020adaptive}). As shown in \cite{roy2020towards}, relying on a priori bounded uncertainty cannot account for state-dependent uncertainty and can even compromise stability.

These limitations motivate the development of an adaptive-robust control framework 
that simultaneously addresses state-dependent uncertainty and time-varying position and velocity constraints—two critical but often separately treated challenges in nonlinear control. 
The key contributions are:

$\bullet$ \textbf{State-dependent uncertainty adaptation}: A novel TDE formulation is derived with an explicit, structured upper bound on the approximation error. The bound's parameters are adapted online, enabling compensation of unknown, state-dependent terms.

$\bullet$ \textbf{Time-varying constraint enforcement}: A BLF-based control strategy is developed to enforce time-varying position and velocity constraints, accommodating 
initial errors and 
transients.

$\bullet$ \textbf{Cohesive design}: The design tightly couples adaptive uncertainty estimation and constraint satisfaction in a provably stable architecture. Stability and constraint satisfaction are established through Lyapunov analysis.

To the best of our knowledge, this is the first framework to couple Adaptive TDE with BLF-based constraint enforcement for nonlinear EL-based robotic systems. The approach is validated on a 5-degrees-of-freedom (DoF) robotic manipulator, outperforming the state of the art. 

The rest of the paper is organised as follows: Section II describes the EL robotics dynamics and the control problem; Section III details the proposed control framework, while corresponding stability analysis is provided in Section IV; comparative experimental results are in Section V, while Section VI provides concluding remarks.

The following notations are used in the paper: $(\bullet)_{L}$ denotes that $(\bullet)(t)$ is delayed by $L$, i.e., $(\bullet)(t-L)$; $\lambda _{\min}(\bullet)$, $\lambda _{\max}(\bullet)$, $|| \bullet ||$, $| \bullet |$ denote minimum and maximum eigenvalue, 2-norm and absolute value of $\bullet$ respectively; $ {I}_n \in\mathbb{R}^{n\times n}$ denotes Identity matrix; $\diag\lbrace \cdot, \cdots, \cdot \rbrace$ denotes a diagonal matrix with diagonal elements $\lbrace \cdot, \cdots, \cdot \rbrace$; $\exp$ denotes exponential function; the saturation function $\sat({s,\varpi})=(s/||s||)$ if $||s||\geq \varpi$ and $\sat({s,\varpi})=(s/\varpi)$ if $||s|| < \varpi$ with $\varpi\in \mathbb{R}^{+}$.

\section{System and Problem Formulation}
Let us consider the following class of robots modeled via Euler-Lagrange (EL) dynamics \cite{7244191, 7506280,   8128491, 8701563, dhadekar2021robust, lim2019delayed}: 
\begin{align} 
   M(   q(t))\ddot{   q}(t) &+   H(  q(t),\dot{  q}(t), t)=   \tau(t), \label{eq:TDdynamics}
\end{align}
where $ {q}, \dot{  q}\in\mathbb{R}^{n}$ are the system states (position and velocity), $  \tau\in\mathbb{R}^{n}$ is the generalized control input, $ {M}\in\mathbb{R}^{n\times n}$ is the mass/inertia matrix and $  H \in\mathbb{R}^{n}$ represents the combined effects of dynamic forces (e.g., Coriolis, dissipative, restoring, external disturbances etc.). 

The following property holds for standard EL mechanics \cite{spong2006robot}:

\textbf{Property 1:} The matrix $ {M}(   q)$ is uniformly positive definite for all $  q$, i.e., $\exists \psi_1, \psi_2 \in \mathbb{R}^{+}$ such that 
\begin{equation}
\psi_1  {I} \leq  {M}(   q) \leq \psi_2  {I} \Rightarrow (1/\psi_2)  {I} \leq  {M}^{-1}(   q) \leq (1/\psi_1)  {I}.\label{prop1}
\end{equation}

Let us introduce the amount of uncertainty in system as follows:
\begin{remark}[Uncertainty]\label{remark_1}
The precise knowledge of $  M$ is not available, only its upper bound ($\psi_2$) is known. 
Whereas, $H$ is unknown for control design.
\end{remark}


Introducing a constant diagonal matrix ${\bar {  M}}$, the dynamics (\ref{eq:TDdynamics}) can be compactly written as 
\begin{align} 
&{\bar{  M}\ddot {  q}} + { {N}}({ {q}},{{\dot {  q}},\ddot{  q}}) = {   \tau },\label{eq:robotdynamics2}\\
\text{with}~~~~&{ {N}}({ {q}},{{\dot {  q}},\ddot{  q}}) = [{ {M}}({ {q}}) - {\bar {  M}}]{\ddot {  q}} +   H(  q,\dot{  q})\label{eq:nonlinearterms}
\end{align}
and the selection of ${\bar {  M}}$ is discussed later (cf. Remark \ref{remark_mass}). The desired trajectories $ {q}^d(t)$ 
and their time-derivatives $\dot{  q}^d(t), \ddot{  q}^d(t) $ are designed to be smooth and bounded. The signals $({ {q}}(t),{{\dot {  q}}(t),\ddot{  q}}(t))$ are available for control design. 

\textbf{Control Objective:} 
Let the position and velocity tracking errors $( {e}, \dot{ {e}})$ be defined as 
\begin{align}\label{error}
        {e} =  {q} -  {q}^d ; 
       ~~~\dot{ {e} } = \dot{ {q}} - \dot{ {q}}^d.  
\end{align}
Let us decompose $e=[e_1~ \ e_2 \ldots e_n]^T$,  $\dot{e}=[\dot{e}_1~ \dot e_2 \ldots \dot{e}_n]^T$, where $e_i$ and $\dot{e}_i$ the tracking errors for the $i^{th}$ degree-of-freedom, $i=1,2, \cdots, n,$. Since we are interested in tracking control, the problem of imposing constraints on states $ {q}, \dot{  q}$ can be equivalently considered to be the problem of imposing constraints on tracking errors $ {e}, \dot{  e}$ around desired trajectories (cf. (\ref{error})). Accordingly, we define constraint functions 

{ \small
\begin{subequations} \label{eq:bounds_p}
\begin{align}
    k_{pi}(t) \hspace{-2pt}&=\hspace{-2pt} (k_{0pi} \hspace{-2pt}-\hspace{-2pt} k_{sspi}) \exp^{-\alpha_{p} t} + k_{sspi}, k_{0pi} \hspace{-2pt}>\hspace{-2pt} {e}_i (0) \\
    k_{vi}(t) \hspace{-2pt}&=\hspace{-2pt} (k_{0vi} \hspace{-2pt}-\hspace{-2pt} k_{ssvi}) \exp^{-\alpha_{v} t} + k_{ssvi},  k_{0vi}\hspace{-2pt}>\hspace{-2pt} \dot{e}_i (0) 
\end{align}
\end{subequations}
}
 where $\alpha_{p},  ~\alpha_{v}  \in \mathbb{R}^{+}$ are scalars; $(k_{0pi}, k_{0vi})$ and $(k_{sspi}, k_{ssvi} )$ are the initial and steady-state values of the respective constraints. Our control objective is to keep the tracking errors within the constraints as $|e_{i}(t)|< k_{pi}(t) , ~|  \dot{e}_{i}(t)|< k_{vi}(t) $ for all time $t \geq 0$. In the following, we discuss  
the design of the constraint functions (\ref{eq:bounds_p}).

\begin{remark}[Design of constraint function]\label{remark_constraint_choice}
To have a well-posed initial setting, initial tracking errors should satisfy $|e_i(0)| < k_{pi}(0), |\dot e_i(0)| < k_{vi}(0)$, as standard in the literature \cite{tee2009barrier, liu2016barrier, cao2022practical, sankaranarayanan2022robustifying, 9636323, xu2021tangent, liu2017barrier, shao2021adaptive,  yang2020state, obeid2018barrier, laghrouche2021barrier, liu2020finite, ding2021adaptive, liu2019adaptive, fuentes2020adaptive}. This can be used to define $(k_{0pi}, k_{0vi})$; whereas, $(k_{sspi}, k_{ssvi})$ can be selected based on user-defined steady-state bounds. Further, the parameters $(\alpha_{pi}, \alpha_{vi})$ determine the convergence rate, which can be as per application requirements. 
\end{remark}


\section{Proposed Adaptive Control Solution}
Variable dependency will be removed subsequently
for brevity whenever it is obvious. The control input $  \tau$ is designed as
\begin{subequations}
\begin{align}
  \tau &={\bar{  M}   u + \hat{  N}}({ {q}},{{\dot {  q}},\ddot{  q}}), \label{eq:input_1}\\ 
\text{with} ~~ {u} &=  {u}_0-  \Delta   u, \label{eq:input_2} \\
\hat{  N}(   q,\dot{  q} , \ddot{  q}) & \cong   N(  q_L,\dot{  q}_L,\ddot{  q}_L)=  \tau_L-\bar{  M}\ddot{  q}_L, \label{eq:approx}\\
  u_0 & = {\ddot {  q}^d} -   D_{D}\dot{  e}-   D_{P}   e, \label{eq:auxinput}\\
{D}_{P}& = \diag{\bigg\{\frac{1}{k_{p1}^{2}- e^{2}_1}, \cdot \cdot \cdot, \frac{1}{k_{pn}^{2}-  e^{2}_n}\bigg\}} , \label{const_gain1} \\
{D}_{D}&=\diag{\bigg\{\frac{1}{k_{v1}^{2}-  \dot{e}^{2}_1}, \cdot \cdot \cdot, \frac{1}{k_{vn}^{2}- \dot{e}^{2}_n}\bigg\}}, \label{const_gain2} 
\end{align}
\end{subequations}
and $  \Delta   u$ is the adaptive control term to be designed later; 
$\hat{ N}$ is the estimate of the uncertainty function $N$ derived from the past state/input data; 
$L>0$ is a small time delay introduced to collect the past data and its choice is discussed later (cf. Remark \ref{remark_L}). \emph{This estimation process via artificial/intentional injection of delay is conventionally called time-delay estimation (TDE)}. 

Substituting (\ref{eq:input_1}), (\ref{eq:input_2}) and (\ref{eq:auxinput}) into (\ref{eq:robotdynamics2}) gives the following error dynamics:
\begin{align}
\ddot{  e}& =-    D_{D}\dot{  e}-   D_{P}   e +   \sigma-  \Delta   u, \label{eq:errdynamics_1}
\end{align}
where $  \sigma={\bar{  M}}^{-1}({  N}-{\hat{  N}})$ is the \textit{estimation error} stemming from TDE process (\ref{eq:approx}) and it is termed as the \textit{TDE error}. The adaptive control term $  \Delta   u$ is designed based on the structure of the upper bound of TDE error $||   \sigma||$ derived subsequently.

\subsection{Upper bound structure of $||   \sigma||$} 
From (\ref{eq:TDdynamics}) and (\ref{eq:errdynamics_1}), the following relations can be achieved:
\begin{align}
\hat{  N}&= {N}_L=[{ {M}}({ {q}_L}) - {\bar {  M}}]{\ddot {  q}_L} +   H_L ,\label{sig 2} \\
   \sigma&=\ddot{  q}-  u. \label{sig 1} 
\end{align}
Using (\ref{sig 2}), the control input $   \tau$ in  (\ref{eq:input_1}) can be rewritten as
\begin{align}
   \tau &= \bar{  M}   u+[{  {M}}({ {q}_L}) - {\bar {  M}}]{\ddot {  q}_L} +   H_L. \label{tau new}
\end{align}
Multiplying both sides of (\ref{sig 1}) with $  M$ and using (\ref{eq:TDdynamics}) and (\ref{tau new})
\begin{align}
 {M}    \sigma  &=    \tau  -  H- {M}  {u}, \nonumber \\
& = \bar{  M}   u+[{  {M}}({ {q}_L}) - {\bar {  M}}]{\ddot {  q}_L} +   H_L -  H - {M}  {u}. \label{sig 3}
\end{align}
Defining 
$ {D} \triangleq [ {D}_P ~  {D}_D]$, $  \xi \triangleq  \begin{bmatrix}
{  e}^T&
\dot{  e}^T
\end{bmatrix}^T$ and using (\ref{eq:errdynamics_1}) we have
\begin{align}
\ddot{ {q}}_L& = \ddot{ {q}}^d_L -   D   \xi_L +    \sigma_L -   \Delta   u_L.\label{sig 4}
\end{align}
Substituting (\ref{sig 4}) into (\ref{sig 3}), and after re-arrangement yields
\begin{align}
   \sigma &= \underbrace{ {M}^{-1}\bar{ {M}}(    \Delta   u_L -    \Delta   u}_{   \chi_1}) + \underbrace{ {M}^{-1}( {M}   \Delta   u -  {M}_L   \Delta   u_L)}_{   \chi_2} \nonumber\\
&+\underbrace{ {M}^{-1}\lbrace (\bar{ {M}} -   {M} )\ddot{ {q}}^d + ( {M}_L - \bar{ {M}})\ddot{ {q}}^d_L+  H_L -  H \rbrace }_{   \chi_3} \nonumber\\
&+ \underbrace{(  I -  {M}^{-1}\bar{ {M}}) {D}   \xi}_{   \chi_4}  + \underbrace{ {M}^{-1}(\bar{ {M}} -  {M}_L) {D}   \xi_L}_{   \chi_5}  \nonumber\\
&+ \underbrace{ {M}^{-1}( {M}_L-\bar{ {M}})   \sigma_L}_{   \chi_6}. \label{sig 5}
\end{align}

Both $ {M}$ and $ {M}^{-1}$ are bounded from system property (\ref{prop1}). 
Using the relation $(\cdot)_L= 
(\cdot)(t)-\int_{-L}^0 {\frac{\mathrm{d}}{\mathrm{d}\theta}(\cdot)}(t+\theta)\mathrm{d}\theta$ and the fact that integration of any continuous function over a finite interval (here $-L$ to $0$) is always finite \cite{rudin1976principles}, the following conditions hold for unknown constants $\delta_i$, $i=1, \cdots,6$:

{\small
\begin{subequations}
\begin{align}
&|| {   \chi_1} ||=||\hspace{-.1cm} - \hspace{-.1cm}  {M}^{-1}\bar{ {M}} \int_{-L}^{0} \frac{\mathrm{d}}{\mathrm{d}\theta}   \Delta {  u}(t+\theta) \mathrm{d}\theta || \leq \delta_1 \label{chi_1}\\
&|| {   \chi_2} ||=|| {M}^{-1}\hspace{-.1cm} \int_{-L}^{0} \frac{\mathrm{d}}{\mathrm{d}\theta} {M}(t+\theta)   \Delta {  u}(t+\theta) \mathrm{d}\theta || \leq \delta_2 \label{chi_2}\\
&|| {   \chi_3} ||=|| {M}^{-1}\lbrace (\bar{ {M}} -   {M} )\ddot{ {q}}^d + ( {M}_L - \bar{ {M}})\ddot{ {q}}^d_L \nonumber\\
&\qquad \qquad \qquad -\int_{-L}^{0} \frac{\mathrm{d}}{\mathrm{d}\theta}{  H}(t+\theta)\mathrm{d}\theta \rbrace ||\leq \delta_3 \label{chi_3}\\
&|| {   \chi_4} ||= || (  I -  {M}^{-1}\bar{ {M}}) {D}   \xi || \leq ||  {E} {D}|| ||   \xi|| \label{chi_5}\\
&|| {   \chi_5} ||=||  {M}^{-1} \int_{-L}^{0} \frac{\mathrm{d}}{\mathrm{d}\theta}( \bar{ {M}} -  {M}(t+\theta)) {D}   \xi(t+\theta) \mathrm{d}\theta  \\
& \qquad \qquad +(  I -  {M}^{-1}\bar{ {M}} ) {D}   \xi  ||\leq ||  E    D ||  ||   \xi || + \delta_5 \label{chi_7}\\
&|| {\chi_6} || \hspace{-2pt}=\hspace{-2pt} ||  E    \sigma +   {M}^{-1}  \hspace{-3pt} \int_{-L}^{0}  \hspace{-3pt} \frac{\mathrm{d}}{\mathrm{d}\theta}\lbrace( {M}(t+\theta)-\bar{ {M}})    \sigma (t+\theta) \rbrace \mathrm{d}\theta || \nonumber \\
& \qquad \leq ||{  E}|| ||   \sigma ||  + \delta_6  . \label{chi_8}
\end{align} 
\end{subequations}
}
Here $ {M}$ and $  H$ are explicitly represented in time for ease of notation. The upper bound of $ ||  \sigma ||$ is formulated using (\ref{chi_1})-(\ref{chi_8}) from (\ref{sig 5}) as
\begin{align}
 \lVert    \sigma \rVert &\leq \beta_{0}^{*} + \beta_{1}^{*} || {D} || \lVert   \xi \rVert, \label{eq:upbound} \\
 \text{where}~ & \beta_{0}^{*} =\frac{\sum_{i=1}^{6}\delta_i}{1-\lVert   E \rVert}, ~ \beta_{1}^{*} = \frac{2 \lVert   E  \rVert  }{1-\lVert   E\rVert}  \nonumber
 \end{align}
under the following condition 
\begin{align} 
 ||   E ||= ||   {I} -  {M}^{-1}\bar{  M} || <1. \label{condition} 
\end{align}
\begin{remark}[Choice of $\bar{  M}$ and upper bound of $||\sigma||$] \label{remark_mass}
 The condition (\ref{condition}) is standard in the TDE literature \cite{8128491, 8701563, dhadekar2021robust, lim2019delayed}: the proposed formulation does not introduce any additional design condition. From the upper bound on $M$ (cf. Remark \ref{remark_1}), one can always design $\bar{  M}$ to satisfy (\ref{condition}).The presence of both states and their constraints in the upper bound structure of $||\sigma||$ requires a new adaptive control framework as derived subsequently.
\end{remark}

\subsection{Adaptive law design} 
The error dynamics (\ref{eq:errdynamics_1}) can be re-written in form of 
\begin{align}
\dot{  \xi}& = (  A_0 +   {\bar{A}})  \xi+   {B} (   \sigma-  \Delta   u), \label{eq:errdynamics_2}
\end{align}
where $ {B}=\begin{bmatrix}
 {0}\\
 {I}_{n}
\end{bmatrix}$,
$  A_0 = \begin{bmatrix}
-    \lambda_P &   0  \\
  0 & -   \lambda_D
\end{bmatrix}$, 
$  {\bar{A}} = \begin{bmatrix}
 \lambda_P &  {I}_n \\
- {D}_P & - {D}_{D}+   \lambda_D
\end{bmatrix}$ and 
$   \lambda_P= \diag\lbrace \lambda_{p1}, \cdots, \lambda_{pn} \rbrace,    \lambda_D = \diag \lbrace \lambda_{d1}, \cdots, \lambda_{dn} \rbrace$ are user-defined positive definite diagonal matrices with $\lambda_{pi}, \lambda_{di} \in \mathbb{R}^{+}$ for each $i=1,2, \cdots,n$.

The term $  \Delta   u$ is designed as
\begin{align} \label{eq:auxinput_2}
  \Delta { {u}}(t) &= \rho \frac{  s}{||  s||} ,
\end{align}
where $ {s=B}^T {D_\xi}  \xi$ and $  D_{\xi =}$ diag$\{ {D}_{P},  {D}_{D} \}$. The adaptive gain $\rho$ in (\ref{eq:auxinput_2}) is formulated based on the upper bound structure of $|| \sigma ||$ from (\ref{eq:upbound}) as
\begin{align}
\rho=\hat{\beta}_0+\hat{\beta}_1 ||  {D}|| ||   \xi || + \zeta, \label{sw gain} 
\end{align}
where $\hat{\beta}_0, \hat{\beta}_1$ are the estimates of $\beta_0, \beta_1 \in \mathbb{R}^{+}$, respectively and adapted via the following laws:
\begin{subequations} \label{adaptive_law}
\begin{align}
&\dot{\hat{\beta}}_{0} = ||  s ||  -    \nu_0\hat{\beta}_{0},~\hat{\beta}_{0} (0) > 0, \label{adaptive_law_1} \\
&\dot{\hat{\beta}}_{1} = ||  s || ||  D|| ||   \xi|| -    \nu_1 \hat{\beta}_{1},~\hat{\beta}_{1} (0) > 0, \label{adaptive_law_2} \\ 
&\dot{\zeta} = 
    -\left( 1 + ||  D_{\xi}||||\bar{  A}||||   \xi||^2 \right. \label{adaptive_law_3} \\
&\left. \quad - \sum_{i=1}^{n} \frac{\dot{k}_{pi}}{k_{pi}} \left[ \frac{e_{i}^{2}} {k_{pi}^{2}-e_{i}^{2}} \right] 
- \sum_{i=1}^{n} \frac{\dot{k}_{vi}}{k_{vi}} \left[ \frac{\dot{e}_{i}^{2}} {k_{vi}^{2}- \dot{e}_{i}^{2}} \right] \right) \zeta + \epsilon, \nonumber
\end{align}
\end{subequations}
where $\nu_i, \epsilon \in\mathbb{R}^{+}$ are user-defined scalars. Finally, combining (\ref{eq:input_1}) and (\ref{eq:auxinput_2}) becomes
\begin{equation}\label{theproposedcontrollaw}
\begin{split}
{  {\tau }} &= \underbrace {{{  {\tau }}_L} - {\bar{  M}}{{{\ddot {  q}}}_L}}_{{\text{TDE~part}}}+\underbrace {{\bar{  M}}({{{\ddot {  q}}}^d} + {{ {D}}_D}{\dot {  e}} + {{ {D}}_P}{ {e}})}_{{\text{Desired~error~constraint~part}}} +\underbrace { \rho {  s}/{||  s||}  }_{{\text{Adaptive~part}}}.
\end{split}
\end{equation}
\begin{remark}[Choice of $L$] \label{remark_L}
The upper bounds in (\ref{chi_1})-(\ref{chi_8}) reveal that high value of time delay $L$ will lead to high values of $\delta_i$, i.e., larger TDE error. Therefore, $L$ is to be selected as the smallest possible sampling time of the low level micro-controller, which is consistent with traditional TDE-based literature \cite{ lim2019delayed, dhadekar2021robust}.
\end{remark}


\begin{remark}
It may appear from (\ref{eq:auxinput}) that the gains $ {D}_{P}$ and $ {D}_{D}$ will become infeasible if the error variables touch the constraints. However, the subsequent closed-loop stability analysis will show that the tracking errors never violate the constraints with the proposed control design.
\end{remark}

\section{Closed-Loop System Stability}\label{sec stability}

\vspace{-4pt}

\begin{theorem}
Under Property 1 and using the proposed control laws (\ref{theproposedcontrollaw}), along with the adaptive law (\ref{adaptive_law}) and the design condition (\ref{condition}), the tracking error trajectories $ {e}$ and $\dot{  e}$ remain within the bounds defined in \eqref{eq:bounds_p} for all $t \geq 0$.
\end{theorem}
\begin{proof}
Modeling the adaptive laws (\ref{adaptive_law_1}) as linear time-varying systems, and using their analytical solutions from positive initial conditions, it can be verified
that $\hat{\beta}_{0}, \hat{\beta}_{1}\geq 0$ and from  (\ref{adaptive_law_3}), it can be verified that $\exists ~\bar{\zeta},  \underline{\zeta} \in \mathbb{R}^{+}$ such that
\begin{align} \label{bound}
0 < \underline{\zeta} \leq \zeta (t) \leq \bar{\zeta}, ~~ \forall t > 0.     
\end{align}

Stability is analyzed using the following Lyapunov function: 
{ \small
\begin{align} \label{lyap}
V &= \frac{1}{2}\sum_{i=1}^{n}\log{\left(\frac{k_{pi}^2}{k_{pi}^{2}-e^{2}_i }\right)} + \frac{1}{2}\sum_{i=1}^{n}\log{\left(\frac{k_{vi}^2}{k_{vi}^{2}-\dot{e}^{2}_i }\right)} \nonumber \\
&+ \sum_{j=0}^{1}\frac{(\hat{\beta}_{j} -\beta_{j}^{*})^2}{2} + \frac{\zeta}{\underline{\zeta}}, 
\end{align}
}

Taking the time derivative of (\ref{lyap}) 
{ \small
 \begin{align} \label{step1}
\dot{V} &= \sum_{i=1}^{n} \left( \frac{e_{i}\dot{e}_{i}}{k_{pi}^{2}-e_{i}^{2}} + \frac{\dot{k}_{pi}}{k_{pi}} - \frac{k_{pi} \dot{k}_{pi}}{ k_{pi}^{2}-e_{i}^{2}} \right) + \sum_{j=0}^{1}(\hat{\beta}_{j} - \beta_{j}^{*}) \dot{\hat{\beta}}_{j} \nonumber \\
&+ \sum_{i=1}^{n} \left( \frac{\dot{e}_{i}\ddot{e}_{i}}{k_{vi}^{2}- \dot{e}_{i}^{2}}  + \frac{\dot{k}_{vi}}{k_{vi}} - \frac{k_{vi} \dot{k}_{vi}}{ k_{vi}^{2}- \dot{e}_{i}^{2}}  \right)  + \frac{\dot{\zeta}}{\underline{\zeta}}  \nonumber \\
&=e^T D_P \dot{e} + \dot{e}^T D_D \ddot{e} +  \sum_{j=0}^{1}(\hat{\beta}_{j} - \beta_{j}^{*}) \dot{\hat{\beta}}_{j} + \frac{\dot{\zeta}}{\underline{\zeta}} \nonumber \\
&+\sum_{i=1}^{n} \left(  \frac{\dot{k}_{vi}}{k_{vi}} - \frac{k_{vi} \dot{k}_{vi}}{ k_{vi}^{2}- \dot{e}_{i}^{2}} + \frac{\dot{k}_{pi}}{k_{pi}} - \frac{k_{pi} \dot{k}_{pi}}{ k_{pi}^{2}-e_{i}^{2}} \right) \nonumber \\
&= \xi^T D_{\xi} \dot{\xi} + \sum_{j=0}^{1}(\hat{\beta}_{j} - \beta_{j}^{*}) \dot{\hat{\beta}}_{j} + \frac{\dot{\zeta}}{\underline{\zeta}}\nonumber \\
&+ \sum_{i=1}^{n} \left(  \frac{\dot{k}_{vi}}{k_{vi}} - \frac{k_{vi} \dot{k}_{vi}}{ k_{vi}^{2}- \dot{e}_{i}^{2}} + \frac{\dot{k}_{pi}}{k_{pi}} - \frac{k_{pi} \dot{k}_{pi}}{ k_{pi}^{2}-e_{i}^{2}} \right). 
\end{align}
}

Using (\ref{eq:errdynamics_2}) one can get
{ \small
\begin{align} \label{step2}
\dot{V} &=     {\xi}^T  {D}_{\xi}    A_0   \xi +    {\xi}^T  {D}_{\xi}   {\bar{A}}   \xi +     {\xi}^T  {D}_{\xi}  {B} (  \sigma-  \Delta   u) + \frac{\dot{\zeta}}{\underline{\zeta}} \nonumber \\ 
& +  \sum_{i=1}^{n} \hspace{-.1cm} \left(  \hspace{-.1cm} \frac{\dot{k}_{vi}}{k_{vi}}\hspace{-.05cm}  -\hspace{-.05cm}  \frac{k_{vi} \dot{k}_{vi}}{ k_{vi}^{2}- \dot{e}_{i}^{2}} \hspace{-.1cm} +\hspace{-.1cm}  \frac{\dot{k}_{pi}}{k_{pi}} - \frac{k_{pi} \dot{k}_{pi}}{ k_{pi}^{2}-e_{i}^{2}} \hspace{-.1cm} \right) \hspace{-.1cm} + \hspace{-.1cm} \sum_{j=0}^{1}(\hat{\beta}_{j} - \beta_{j}^{*}) \dot{\hat{\beta}}_{j} \nonumber \\
&\leq    {\xi}^T  {D}_{\xi}    A_0   \xi + ||  D_{\xi}||||\bar{  A}||||   \xi||^2 + ||s|||| \sigma|| - \rho||s|| + \frac{\dot{\zeta}}{\underline{\zeta}} \nonumber \\
& +  \sum_{i=1}^{n} \hspace{-.1cm} \left(\hspace{-.1cm}   \frac{\dot{k}_{vi}}{k_{vi}}\hspace{-.05cm}  - \hspace{-.05cm} \frac{k_{vi} \dot{k}_{vi}}{ k_{vi}^{2}- \dot{e}_{i}^{2}} \hspace{-.05cm} + \hspace{-.05cm} \frac{\dot{k}_{pi}}{k_{pi}} - \frac{k_{pi} \dot{k}_{pi}}{ k_{pi}^{2}-e_{i}^{2}} \hspace{-.1cm} \right)\hspace{-.1cm}  +\hspace{-.1cm}  \sum_{j=0}^{1}\hspace{-.05cm} (\hat{\beta}_{j} - \beta_{j}^{*}) \dot{\hat{\beta}}_{j} \nonumber \\
&\leq      {\xi}^T  {D}_{\xi}    A_0   \xi + ||  D_{\xi}||||\bar{  A}||||   \xi||^2 - (\hat{\beta}_{0} - \beta_{0}^{*}) ||s|| \nonumber \\
&- (\hat{\beta}_{1} - \beta_{1}^{*}) (||  s|| ||   \xi||||  D||)  + \sum_{j=0}^{1}(\hat{\beta}_{j} - \beta_{j}^{*}) \dot{\hat{\beta}}_{j}  + \frac{\dot{\zeta}}{\underline{\zeta}} \nonumber \\
&+  \sum_{i=1}^{n} \left(  \frac{\dot{k}_{vi}}{k_{vi}} - \frac{k_{vi} \dot{k}_{vi}}{ k_{vi}^{2}- \dot{e}_{i}^{2}} + \frac{\dot{k}_{pi}}{k_{pi}} - \frac{k_{pi} \dot{k}_{pi}}{ k_{pi}^{2}-e_{i}^{2}} \right).
\end{align}
}
Also, from (\ref{adaptive_law_1})  we have
{\small
 \begin{align} \label{step3}
 &\sum_{j=0}^{1}(\hat{\beta}_{j} - \beta_{j}^{*}) \dot{\hat{\beta}}_{j} = (\hat{\beta}_{0} - \beta_{0}^{*}) (||s|| || -    \nu_0 \hat{\beta}_{0}) \nonumber \\ 
 & +(\hat{\beta}_{1} - \beta_{1}^{*}) (||s|| ||\xi|| ||D|| -    \nu_1 \hat{\beta}_{1}) \nonumber \\ 
 &= (\hat{\beta}_{0} - \beta_{0}^{*}) ||s||  +(\hat{\beta}_{1} - \beta_{1}^{*}) (||  s|| ||   \xi||||  D||) \nonumber \\ 
 & +  \sum_{j=0}^{1}(   \nu_j \hat{\beta}_{j} \beta_{j}^{*} -    \nu_j \hat{\beta}_{j}^2).  
 \end{align}
 }
 One can verify that
 {\small
 \begin{equation} \label{step4}
 (   \nu_j \hat{\beta}_{j} \beta_{j}^{*} -    \nu_i \hat{\beta}_{j}^2) = - \frac{   \nu_i}{2}  \left((\hat{\beta}_{j} -  \beta_{j}^*)^2 - {\beta_{j}^*}^2\right). 
 \end{equation}
}
For $l= p,v$ and $\mu = e, \dot{e}$, we can simplify
\begin{align} \label{step5}
&\sum_{i=1}^{n} \frac{\dot{k}_{li}}{k_{li}} - \frac{k_{li} \dot{k}_{li}}{ k_{li}^{2}-\mu_{i}^{2}}
=
-\sum_{i=1}^{n} \frac{\dot{k}_{li}}{k_{li}} \left[ \frac{\mu_{i}^{2}} {k_{li}^{2}-\mu_{i}^{2}} \right] 
 \end{align}
The adaptive law (\ref{adaptive_law_3}) and relation (\ref{bound}) lead to
{\small
 \begin{align} \label{step6}
 \frac{\dot{\zeta}}{\underline{\zeta}} & \leq  -||  D_{\xi}||||\bar{  A}||||   \xi||^2 +\sum_{i=1}^{n} \frac{\dot{k}_{pi}}{k_{pi}} \left[ \frac{e_{i}^{2}} {k_{pi}^{2}-e_{i}^{2}} \right] \nonumber \\
&+\sum_{i=1}^{n} \frac{\dot{k}_{vi}}{k_{vi}} \left[ \frac{\dot{e}_{i}^{2}} {k_{vi}^{2}- \dot{e}_{i}^{2}} \right] + \frac{{{\epsilon}}}{\underline{\zeta}}.
 \end{align}
 }
Using results from \eqref{step3},\eqref{step4},\eqref{step5},\eqref{step6} into \eqref{step2}, give
{ \small
\begin{align} \label{step7}
\dot{V} &=    {\xi}^T  {D}_{\xi}    A_0   \xi  -\sum_{j=0}^{1}\frac{   \nu_i}{2}  \left((\hat{\beta}_{j} -  \beta_{j}^*)^2 - {\beta_{j}^*}^2\right) + \frac{{{\epsilon}}}{\underline{\zeta}} \nonumber \\
&= - \begin{bmatrix}
{  e}^T&
\dot{  e}^T
\end{bmatrix} \begin{bmatrix} D_p \lambda_p &  0 \\ 0 & D_d \lambda_d \end{bmatrix} \begin{bmatrix}
{  e} \nonumber \\
\dot{  e} \end{bmatrix} \nonumber \\
&-\sum_{j=0}^{1}\frac{   \nu_i}{2}  \left((\hat{\beta}_{j} -  \beta_{j}^*)^2 - {\beta_{j}^*}^2\right) + \frac{{{\epsilon}}}{\underline{\zeta}} \nonumber \\
&= -\sum_{i=1}^{n}  
\lambda_{pi} {\left(\frac{e_{i}^2}{k_{pi}^{2}-e^{2}_i }\right)}  -\sum_{i=1}^{n}  
\lambda_{di} {\left(\frac{\dot{e}_{i}^2}{k_{vi}^{2}-\dot{e}^{2}_i }\right)} \nonumber \\
&-\sum_{j=0}^{1}\frac{   \nu_j}{2}  \left((\hat{\beta}_{j} -  \beta_{j}^*)^2 - {\beta_{j}^*}^2\right) + \frac{{{\epsilon}}}{\underline{\zeta}}.
\end{align}
}
From the fact that
$   \log\left( \frac{k_x^2}{k_x^2 - x^2} \right) < \frac{x^2}{k_x^2 - x^2}, \quad \text{for } |x| < k_x$,
which holds within any compact set \( \Omega : |x| < k_x\), and for any \( k_x \in \mathbb{R}_+ \), we can write
{\small
\begin{align} \label{step8}
    \dot{V} \leq &-\sum_{i=1}^{n}  
\lambda_{pi} \log{\left(\frac{k_{pi}^2}{k_{pi}^{2}-e^{2}_i }\right)}  -\sum_{i=1}^{n}  
\lambda_{di} \log{\left(\frac{k_{v`i}^2}{k_{vi}^{2}-\dot{e}^{2}_i }\right)} \nonumber \\
&-\sum_{j=0}^{1}\frac{   \nu_j}{2}  \left((\hat{\beta}_{j} -  \beta_{j}^*)^2 - {\beta_{j}^*}^2\right) + \frac{{{\epsilon}}}{\underline{\zeta}}.
\end{align}
}
 The definition of $V$ as in (\ref{lyap}) yields
 {\small
 \begin{align} \label{step9}
V &= \frac{1}{2}\sum_{i=1}^{n}\log{\left(\frac{k_{pi}^2}{k_{pi}^{2}-e^{2}_i }\right)} + \frac{1}{2}\sum_{i=1}^{n}\log{\left(\frac{k_{vi}^2}{k_{vi}^{2}-\dot{e}^{2}_i }\right)} \nonumber \\
&+ \sum_{j=0}^{1}\frac{(\hat{\beta}_{j} -\beta_{j}^{*})^2}{2} + \frac{\bar{\zeta}}{\underline{\zeta}}.
\end{align}
}
Combining $\dot{V}$ and $V$  from \eqref{step8} and \eqref{step9}
{\small
 \begin{align} \label{subs_10}
\dot{V} \leq -\varrho V  +  \varrho \frac{\bar{\zeta}}{\underline{\zeta}} + \frac{{{\epsilon}}}{\underline{\zeta}}  + \frac{1}{2}\sum \limits_{i=0}^{1}    \nu_i {\beta^*_{i}}^2
 \end{align}
 }
where $\varrho \triangleq \min \lbrace \min\lbrace \lambda_{pi} \rbrace, \min\lbrace \lambda_{di} \rbrace, {\nu_{pi}/2} \rbrace$ .

\end{proof}

\begin{remark}[Continuity in control law]
To make the control laws continuous for practical implementation, the terms $ {(s/||s||)}$  are usually replaced by continuous saturation functions  $\sat( {s}, \varpi)$ for $\varpi \in \mathbb{R}^{+}$ (cf. the notation definition in Sect. I): this still keeps the closed-loop system bounded with minor modification in stability analysis (cf. \cite{roy2019overcoming}), and hence repetition is avoided.
\end{remark}

\begin{algorithm}[!h]
 {\caption{ {Design steps of the proposed controller}}
 \vspace{0.0cm}
{\textbf{Step 1 (Selection of the constraints)}: Select the values of various parameters of the time-varying constraints, i.e., of the steady-state bounds $(k_{sspi}, k_{ssvi})$, of the initial values $(k_{0pi}, k_{0vi})$, and of the parameters $(\alpha_{pi}, \alpha_{vi})$ as per requirements, following Remark \ref{remark_constraint_choice}.
 
\textbf{Step 2 (Defining control gains):} Define state-constraint gains $ {D}_{P}$ and $ {D}_{D}$ based upon the parameters selected in Step 1.

\textbf{Step 3 (Designing adaptive controller):} 
Define $\bold{s}$, gains $\lambda_P, \lambda_D$ and $\rho$ as in (\ref{sw gain}) using adaptive laws in (\ref{adaptive_law})}}.

\textbf{Step 4 (Defining control input):} Select $\bar{  M}$ using \eqref{condition}  
and use the values chosen in Steps 1-3 to design the control inputs $  {\tau}$ as given in \eqref{theproposedcontrollaw}.

\vspace{0.0cm}
 \end{algorithm}

\section{Experimental Validation and Results}

The experimental platform consisted of UFactory 5-DoF xArm-5 robotic manipulator (cf. Fig. \ref{fig:xarm}), with an NVIDIA Jetson AGX Xavier serving as the compute unit. A custom-designed end-effector was attached, incorporating a Dynamixel XM430-W210-T servo motor to enable rapid switching between two tools: a marker and an eraser. This configuration allows for seamless tool transitions during task execution. Real-time joint state data were acquired via the manipulator’s internal feedback interface, supporting synchronized, high-frequency closed-loop control.

\subsection{Experimental Scenario}
The robot was commanded to perform a constrained drawing and erasing task on a surface: drawing of $3$ layers of concentric semicircles and then erase the middle semicircle without erasing/ disturbing the inner and outer semicircles. Involving physical contact and tight position and velocity limits in the task was designed to evaluate tracking performance and constraint handling under state-dependent and external disturbances arising from unknown surface contact forces and frictional variations at the end-effector. The specific sequences of the task are as follows:
\begin{itemize}
    \item Starting from the initial position, draw an innermost semicircle with a radius of $20$ cm; lift the pen and come back to the initial position.
    \item Draw an outermost semicircle with a radius of $24$ cm; lift the pen and come back to the initial position.
    \item Draw a middle semicircle with a radius of $22$ cm; lift the pen and come back to the initial position.
    \item Rotate the tool (from drawing to erasing) and erase the middle ($22$ cm) semicircle using the eraser starting from the initial position.
\end{itemize}
Both the drawing and erasing of the middle semicircle occur in a confined region between the inner and outer paths, thereby emphasizing the critical role of real-time constraint enforcement to ensure safe and feasible motion. The desired trajectory $({  q}^d, \dot{  q}^d)$ for the manipulator is pre-computed using the open-source package named \textit{ruckig}  \cite{berscheid2021jerk}. The $4$ cm radial distance between the inner and outer semicircles was selected to accommodate the $2$ cm wide eraser. It is important to note that the initial condition of the robot is always set at zero (as provided by the manufacturer). Therefore, the initial error always starts from zero (cf. the experimental error plots later). 

To evaluate the significance of the proposed control strategy, its performance was compared against two representative baselines: (i) an adaptive BLF-based controller \cite{liu2020finite} (referred as ABLF) and (ii) an adaptive TDE-based controller \cite{9001188} (referred ATDC).
The parameters used in the experiments for the proposed controller are listed in Table \ref{tab:control_parameters} along with parameters to design the constraints.
For both baselines, the controller gains were carefully tuned according to the respective design guidance to achieve their best possible performance under identical task conditions.

\begin{figure}
\includegraphics[width=.45\textwidth]{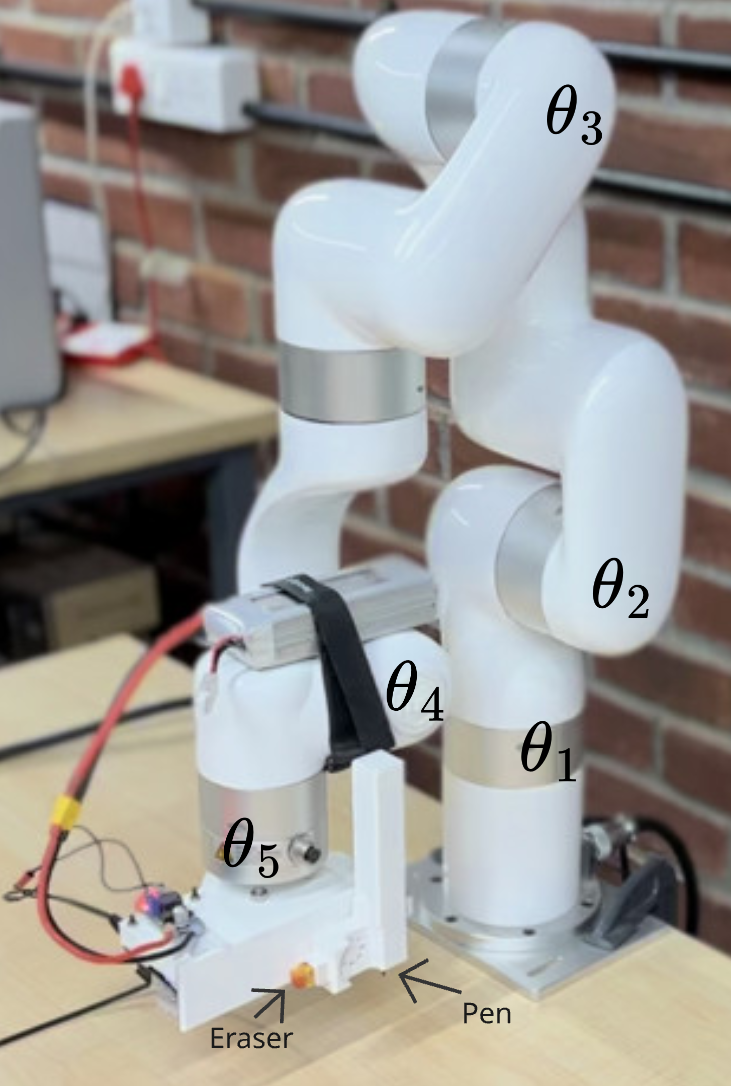}
    \centering
    \caption{Experimental setup of 5-DoF xArm and custom end-effector.}
    \label{fig:xarm}
\end{figure}

\begin{table}[htbp]
\centering
\caption{\color{black} Design Parameters for the Proposed Controller}
\begin{tabular}{|p{0.45\textwidth}|}
\hline
\vspace{0.1mm}
${\bar{  M}} =    I$;  $\lambda_D = \text{diag}\{3, 3, 3, 3, 3\}$ ;
$\lambda_P = \text{diag}\{5, 5, 5, 5,5\}$ \\
$L = 0.01$ \\
$\hat{\beta}_0(0) = \hat{\beta}_1(0)  = 0.01$; $\nu_{0} = \nu_{1} = 10.0$ \\
$\zeta(0) = 0.1$ ;
${\epsilon} = 0.0001$ \\
$k_{0p1}=k_{0p2}=k_{0p3}=k_{0p4}=k_{0p5}= 10.0$ \\
$k_{0v1}=k_{0v2}=k_{0v3}=k_{0v4}=k_{0v5}=20.0$ \\ 
$k_{ssp1}=k_{ssp2}=k_{ssp3}=k_{ssp4}=k_{ssp5}= 5.0$ \\ 
$k_{ssv1}=k_{ssv2}=k_{ssv3}=k_{ssv4}=k_{ssv5}= 10.0$ \\ 
$\alpha_{p} = \alpha_{v} = 0.075$ \\

\hline
\end{tabular}
\label{tab:control_parameters}
\end{table}

\subsection{Results and Analysis}
The performance of the controllers is illustrated in Figs. \ref{fig:exp_result}–\ref{fig:velocity_error}. Figure \ref{fig:exp_result} clearly illustrates that the semicircles drawn by the proposed controller closely match the intended trajectories, while those produced by the baseline controllers exhibit noticeable distortions. 
The results clearly demonstrate that the proposed controller outperforms the baselines in both tracking accuracy and constraint satisfaction. As shown in the tracking error plots \ref{fig:angle_error}–\ref{fig:velocity_error}, it consistently keeps the tracking errors within the prescribed position and velocity bounds—a direct result of its integrated adaptive time-delay estimation and BLF-based constraint handling. 

These qualitative trends are further substantiated by the quantitative results in Table \ref{table:result}, where the proposed controller achieves the lowest Root Mean Square (RMS) errors across all joints. ABLF, equipped with constraint handling capability, could honor the constraints; however, it cannot negotiate state-dependent and unmodelled uncertainties, leading to distorted drawings. Although ATDC can tackle state-dependent uncertainty, it lacks the ability to handle constraint; consequently, ATDC partially erases the inner and outer semicircles, which is unintended. These observations highlight that the capability to handle either the state-dependent uncertainty or the state constraints alone is not sufficient; a control framework, like the proposed one, must combine both these features to achieve the desired result. 

\begin{figure}
  \includegraphics[width=.46\textwidth]{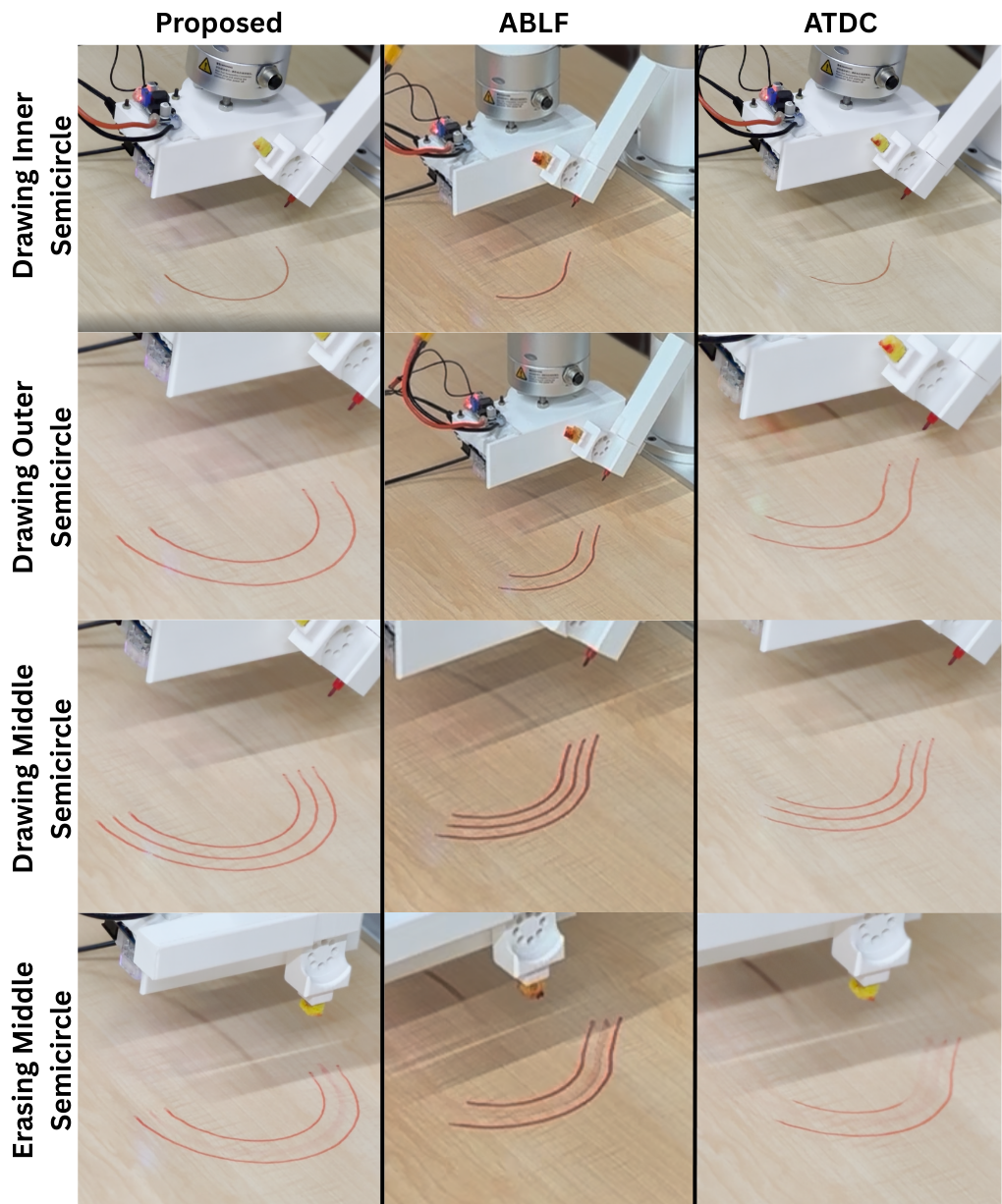}
    \centering
    \caption{Controller comparison during the constrained drawing–erasing task. Columns: Proposed, ABLF, and ATDC. Rows: drawing inner semicircle, drawing outer semicircle, drawing middle semicircle, and erasing the middle semicircle.}
    \label{fig:exp_result}
\end{figure}

\begin{figure}
    \includegraphics[width=.46\textwidth]{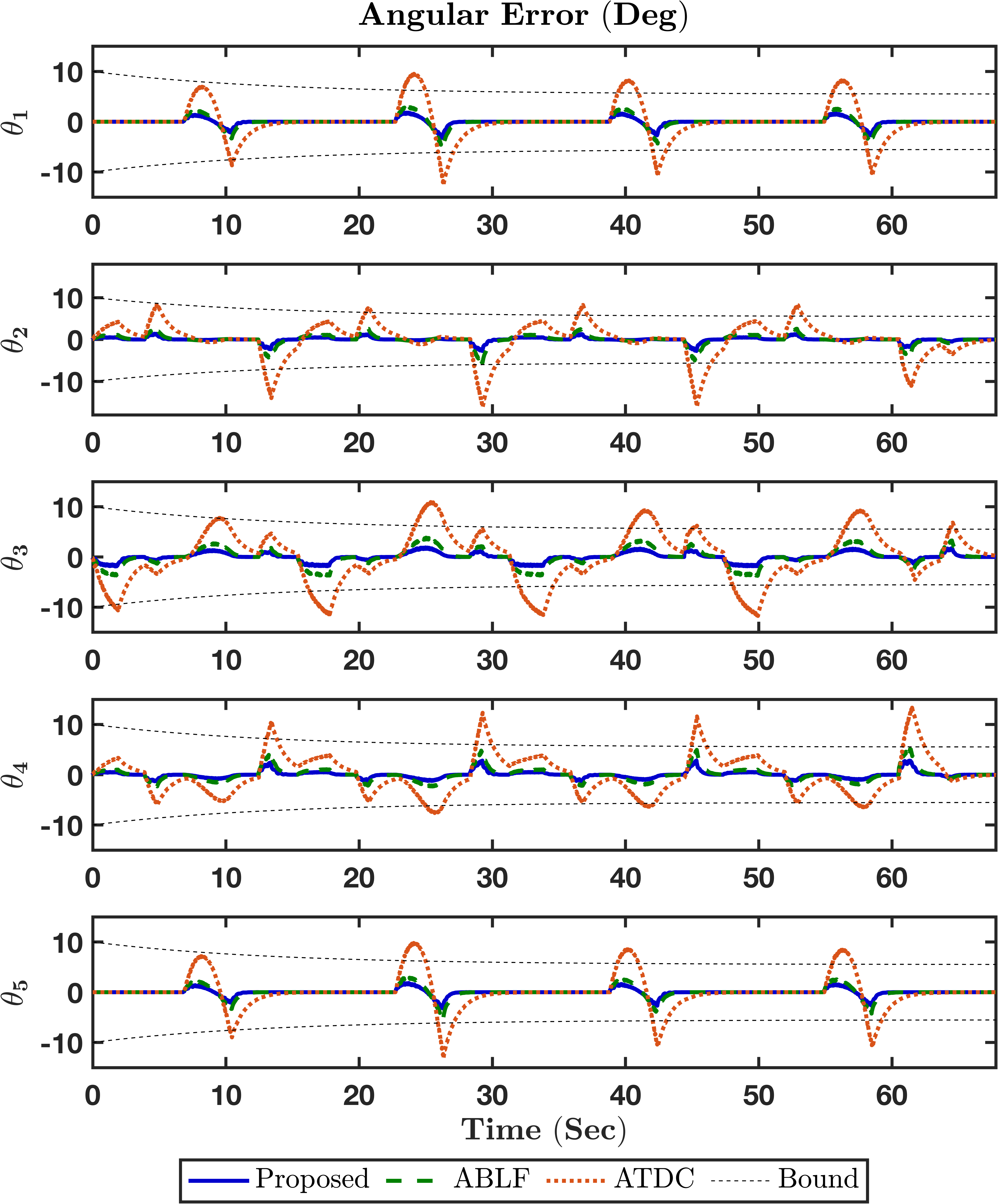}
    \centering
    \caption{Joint angle tracking error performance comparison.}
    \label{fig:angle_error}
\end{figure}

\begin{figure}
\includegraphics[width=.46\textwidth]{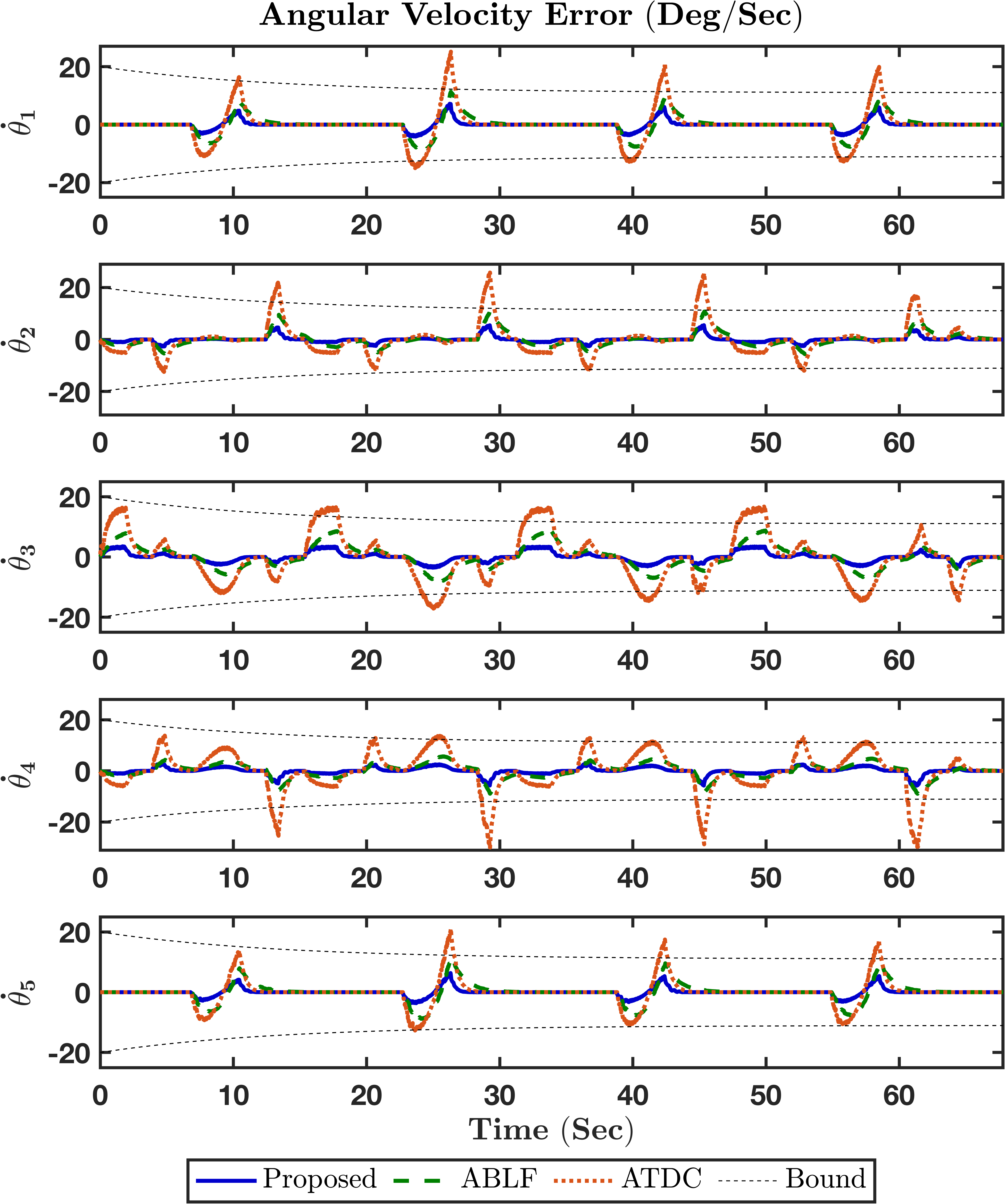}
    \centering
    \caption{Joint velocity tracking errors performance comparison.}
    \label{fig:velocity_error}
\end{figure}

\begin{table}[!t]
\renewcommand{\arraystretch}{1.1}
\caption{Performance Comparison in RMS Error}
\label{table:result}
\centering
\resizebox{\columnwidth}{!}{
\begin{tabular}{c c c c c c c}
    \hline\hline
    & & $\theta_1$ & $\theta_2$ & $\theta_3$ & $\theta_4$ & $\theta_5$ \\
    \hline
   \multirow{3}{*}{Position error (deg)} & Proposed & 0.64 & 1.06 & 0.82 & 0.98 & 0.63 \\
   & ABLF & 1.06 & 1.75 & 1.41 & 1.63 & 1.04 \\
   & ATDC & 2.31 & 4.04 & 3.63 & 3.91 & 2.30 \\
    \hline
    \multirow{3}{*}{Velocity error (deg/s)} & Proposed & 2.40 & 2.19 & 3.10 & 2.24 & 1.79 \\
    & ABLF & 3.88 & 4.26 & 3.82 & 4.12 & 3.87 \\
    & ATDC & 4.06 & 6.72 & 6.31 & 7.32 & 4.01 \\
    \hline\hline
\end{tabular}
}
\end{table}

\section{Conclusions}
This paper presented a unified adaptive-robust control framework that integrates time-delay estimation with BLF-based constraint enforcement for Euler–Lagrange robotic systems under state-dependent uncertainties. The proposed approach achieves real-time tracking while rigorously satisfying time-varying position and velocity constraints—without requiring prior model knowledge. Experimental validation on a 5-DoF robotic manipulator demonstrated the superior performance of the proposed controller over the state-of-the-art in contact-rich tasks, achieving lower tracking errors and strict constraint compliance. A future work would be to extend the proposed framework for underactuated robotic systems. 

\bibliographystyle{IEEEtran}
\bibliography{root}
\vspace{-1px}

\end{document}